\documentclass[conference]{IEEEtran}
\IEEEoverridecommandlockouts
\usepackage{cite}
\usepackage{amsmath,amssymb,amsfonts}
\usepackage{textcomp}
\usepackage{xcolor}
\usepackage{verbatim}
\usepackage{graphicx,subfigure}\allowdisplaybreaks[4] 
\usepackage{algorithm} 
\usepackage{algorithmic}
\usepackage{authblk}
\usepackage{bm}
\def\BibTeX{{\rm B\kern-.05em{\sc i\kern-.025em b}\kern-.08em
    T\kern-.1667em\lower.7ex\hbox{E}\kern-.125emX}}

\title{Over-the-Air Federated Multi-Task Learning} 

\author[$\dag$]{Haoming Ma}  
\author[$\dag$]{Xiaojun Yuan}  
\author[$\dag$]{Dian Fan}  
\author[$\ddag$]{Zhi Ding}  
\author[$\sharp$]{Xin Wang}
\author[$\dag$]{Jun Fang} 
\affil[$\dag$]{\normalsize The National Key Lab. of Sci. and Tech. on Commun., Uni. of Electronic Sci. and Tech. of China, Chengdu, China}  
\affil[$\ddag$]{Department of Electrical and Computer Engineering, University of California at Davis, Davis, United States}  
\affil[$\sharp$]{The Key Laboratory for Information Science of Electromagnetic Waves (MoE), Fudan University, Shanghai, China} 

\begin{document}

\maketitle

\begin{abstract}
In this letter, we introduce over-the-air computation into the communication design of federated multi-task learning (FMTL), and propose an over-the-air federated multi-task learning (OA-FMTL) framework, where multiple learning tasks deployed on edge devices share a non-orthogonal fading channel under the coordination of an edge server (ES). Specifically, the model updates for all the tasks are transmitted and superimposed concurrently over a non-orthogonal uplink fading channel, and the model aggregations of all the tasks are reconstructed at the ES through a modified version of the turbo compressed sensing algorithm (Turbo-CS) that overcomes inter-task interference. Both convergence analysis and numerical results show that the OA-FMTL framework can significantly improve the system efficiency in terms of reducing the number of channel uses without causing substantial learning performance degradation.
\end{abstract}


\section{Introduction}
With a massive amount of data at wireless edge devices, federated learning (FL) \cite{goetz_active_2019} has emerged as a popular framework for training machine learning models in a confidential and distributive manner. In general, FL requires uploading local model parameters from devices to a specific edge server (ES). Owing to massively distributed devices as well as a huge amount of model parameters, limited channel resource of uplink communication (e.g., bandwidth, time and space) poses a major bottleneck in the original FL framework. To this end, extensive research effort has recently been devoted to enhance communication-efficiency in FL. For example, the authors in \cite{sattler_sparse_2019,lin_deep_2020} proposed to relieve the uplink burden by sparsifying and compressing the local model updates before transmission. In \cite{amiri_machine_2020,amiri_federated_2020,liu_reconfigurable_2021}, over-the-air computation was used to speed up local model aggregation by exploiting radio superposition over shared physical channels. 

Based on the basic idea of multi-task learning \cite{zhang_survey_2021} and FL, the authors in \cite{smith_federated_2018,dinh_fedu_2021} proposed the federated multi-task learning (FMTL) framework to implement multiple machine learning tasks over the FL network, so that the knowledge contained in a task can be leveraged by other tasks with the hope of improving the generalization performance \cite{zhang_survey_2021}. Despite the appealing aspects of FMTL, the inter-task interference inherent in FMTL hinders the direct implementation of existing transmission protocols \cite{sattler_sparse_2019,lin_deep_2020,amiri_machine_2020,amiri_federated_2020,liu_reconfigurable_2021} designed for FL over wireless networks. To overcome the inter-task interference, a straightforward extension involves separating the uplink transmission for multi-task updates over orthogonal frequency/time sub-channels. However, this frequency/time division approach may be inefficient since the overall channel resource is divided into orthogonal ones to avoid the inter-task interference. By contrary, in this paper, we propose a novel non-orthogonal transmission scheme in the presence of inter-task interference, where the local updates for all the tasks are sent simultaneously over the same fading channel to achieve communication-efficient FMTL.

To be more specific, we investigate the over-the-air FMTL (OA-FMTL) transmission scheme, where multiple tasks deployed on edge devices share a non-orthogonal fading channel under the coordination of an ES. At every edge device, the local model updates of all tasks deployed on devices are first sparsified and compressed individually by following the approach of \cite{sattler_sparse_2019,lin_deep_2020}, prior to being transmitted and aggregated over the uplink channel. Model aggregations of the individual tasks are reconstructed efficiently at the ES by exploiting a novel modified version of the turbo compressed sensing (Turbo-CS) algorithm \cite{ma_turbo_2014}. State evolution and convergence analysis are established to characterize the behavior of the proposed OA-FMTL scheme. Experimental simulations show that our proposed OA-FMTL is able to achieve a learning performance comparable to the signal-task transmission scheme \cite{amiri_machine_2020,amiri_federated_2020} by efficiently suppressing the inter-task interference. In other words, the communication resource required by OA-FMTL is only one \emph{N}-th of the conventional frequency/time division approach, where \emph{N} is the total number of tasks.

\section{System Model}
\subsection{Federated Multi-Task Learning}\label{sec:system,ssec:multi}
We consider an FMTL system with $N$ learning tasks deployed on $M$ wireless local devices with the help of an ES, where the practical task assignment is flexibly determined according to the computation power and storage capability of each device, as depicted in Fig. \ref{sec:system,ssec:multi,fig:system}. Each task $n$ on device $m$ is associated with its local dataset $D_{nm}$. The FMTL requires the minimization of the total empirical loss function\footnote{Different from the FMTL schemes in \cite{smith_federated_2018} and \cite{dinh_fedu_2021}, we omit a regularization term in the federated optimization objective. The optimization associated with this term is undertaken by the ES independently and thus is irrelevant to the design of communications between the devices and the ES.}, defined as the sum of the losses of the $N$ tasks,
\begin{equation}\label{sec:system,ssec:multi,equ:optimum function}
	\mathcal{L}(\bm{\theta})=\sum_{n=1}^{N} \mathcal{L}_{n}(\bm{\theta}_{n}),
\end{equation}
where $\bm{\theta}=[{\bm{\theta}_1}^T,\dots,{\bm{\theta}_N}^T]^T$ with $\bm{\theta}_{n}\in\mathbb{R}^{d_n}$ being the length-$d_n$ model parameter of task $n$ shared among the ES and participating edge devices; and the empirical loss function of each task $n$ is defined by
\begin{equation}\label{sec:system,ssec:multi,equ:loss function}	
	\mathcal{L}_n(\bm{\theta}_{n})=\frac{\sum_{m=1}^{M} K_{nm} \mathcal{L}_{nm}(\bm{\theta}_{n})}{\sum_{m=1}^{M} K_{nm}},
\end{equation}
with the local empirical loss function of task $n$ on device $m$ defined by 
\begin{equation}
	\mathcal{L}_{nm}(\bm{\theta}_{n})=
	\left\{
	\begin{array}{ll} 			
		\frac{1}{K_{nm}}\sum_{k=1}^{K_{nm}}l_n(\bm{\theta}_{n};\bm{u}_{nmk}), & \text{for } K_{nm}\neq 0,\\ 		
		0, & \text{for } K_{nm}= 0, 		
	\end{array}
	\right.	 	
\end{equation}  
where $l_n(\bm{\theta}_{n}; \bm{u}_{nmk})$ is the sample-wise loss function specified by task $n$, $\mathbf{u}_{nmk}$ denotes the $k$-th local data sample of dataset $D_{nm}$, $K_{nm}$ is the cardinality of $D_{nm}$, and $[k]$ denotes the integer set $\{1,\dots,k\}$. Note that $K_{nm}=0$ means that $D_{nm}$ is empty. 	

\begin{figure}[h] 	
	\centering 	
	\includegraphics[width=0.78\linewidth, height=0.2\textheight]{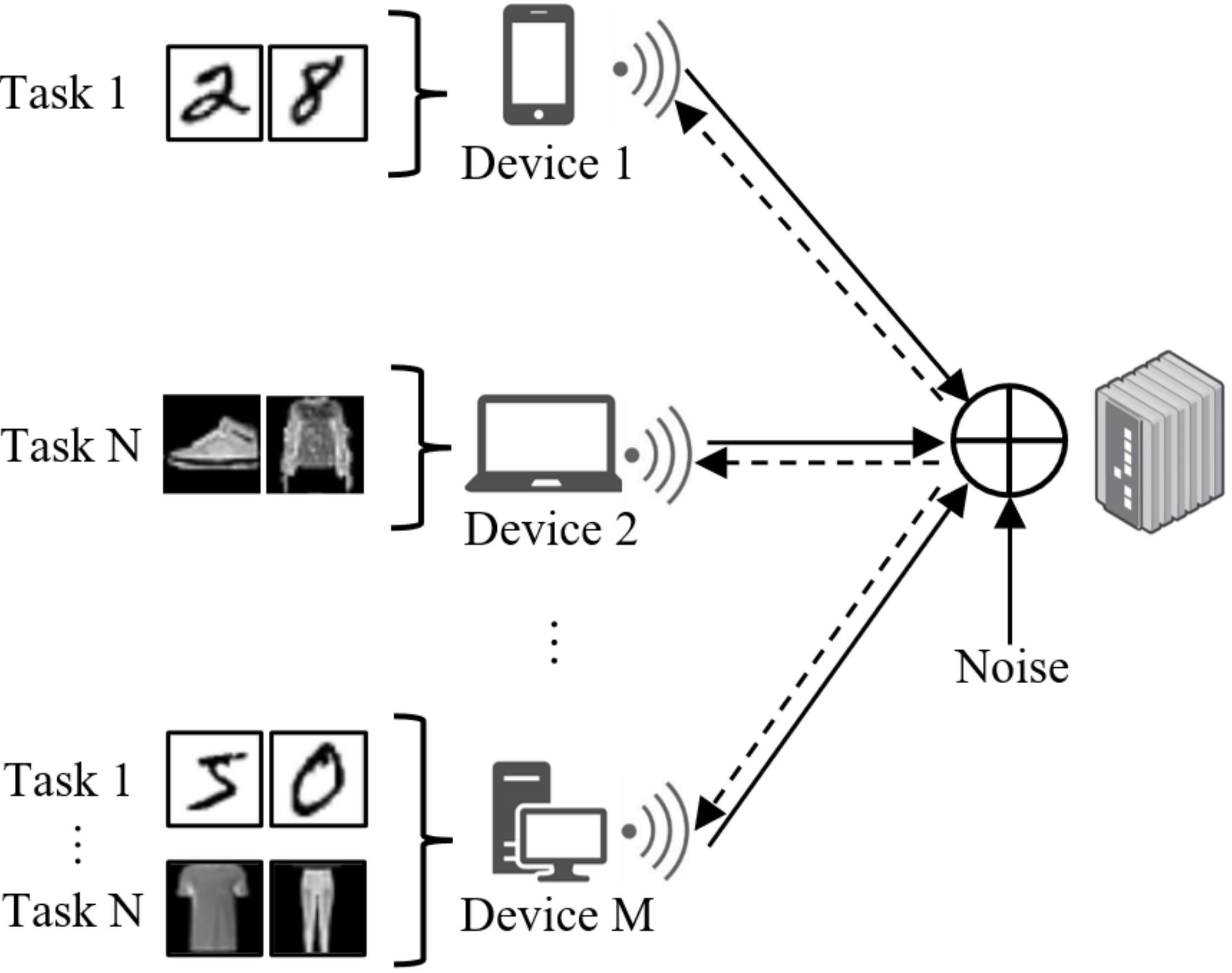} 	
	\caption{An illustration of the FMTL framework.} 	\label{sec:system,ssec:multi,fig:system}
 \end{figure}

The minimization of (\ref{sec:system,ssec:multi,equ:optimum function}) is typically executed through gradient-based update, i.e., at the $t$-th communication round, the global model parameter $\bm{\theta}^{(t)}$ is expected to update via
\begin{equation}\label{update}
	\bm{\theta}^{(t+1)}=\bm{\theta}^{(t)}-\eta\nabla \mathcal{L}(\bm{\theta}^{(t)}),
\end{equation}
where $\nabla=\begin{pmatrix}
	\nabla_1^T,\dots,\nabla_N^T
\end{pmatrix}^T$, the gradient operator $\nabla_n$ is with respect to the segment $\bm{\theta}_n$, and $\eta$ is the predetermined learning rate. Combining (\ref{sec:system,ssec:multi,equ:optimum function}) and (\ref{update}), the parameter segment $\bm{\theta}_{n}^{(t)}$ of each task $n$ is expected to be updated via
\begin{subequations}\label{model_update} 
	\begin{align}		\bm{\theta}_{n}^{(t+1)}&=\bm{\theta}_{n}^{(t)}-\eta\nabla_n \mathcal{L}_n(\bm{\theta}_{n})\\ 		\nonumber&=\bm{\theta}_{n}^{(t)} - \eta\frac{\sum_{m=1}^{M} K_{nm}\bm{g}_{nm}^{(t)}}{\sum_{m=1}^MK_{nm}}, \forall n \in [N],\tag{5b} 	
	\end{align}
\end{subequations}
where the local gradient $\bm{g}_{nm}^{(t)}=\nabla_n \mathcal{L}_{nm}(\bm{\theta}_{n})\in\mathbb{R}^{d_n}$. In practice, the local gradients $\{\bm{g}_{nm}^{(t)}\}_{n=1}^N$ from each device $m$ are sent to the ES over a wireless uplink to complete the updating of model parameter as in (\ref{model_update}) subject to some transmission error. After model updating, $\{\bm{\theta}_{n}^{(t+1)}\}_{n=1}^N$ are broadcast to all the devices by the ES over a wireless downlink to synchronize the learning tasks among the devices. The iteration process in (\ref{model_update}) continues until the learning tasks converge.

\subsection{Over-the-Air Channel Model}
We now describe the wireless channels used to support the above FMTL process. Following the convention in \cite{amiri_machine_2020,amiri_federated_2020,liu_reconfigurable_2021}, we assume that the downlink transmission from the ES to the devices is error-free, and focus on the uplink. We model the wireless uplink as a block fading channel with the channel state information unchanged within each communication round. Noting that the update in (\ref{model_update}) depends exclusively on the weighted sum of the local gradients, we employ the over-the-air computation to reduce the usage of channel resources. Specifically, at the $t$-th communication round, in an analog fashion, every device equipped with an individual antenna synchronously sends its channel input vector to the ES over a block fading channel with $s$ frequency/time channel uses (with $s\le d$), characterized by:
\begin{equation}\label{sec:system,ssec:trans,equ:channel}
	\bm{r}^{(t)}=\sum_{m=1}^{M}h_m^{(t)}\bm{s}_m^{(t)} + \bm{w},	
\end{equation}
where $\bm{s}_m^{(t)}\in \mathbb{C}^s$ is the channel input vector from device $m$, with the details specified later in Section \ref{sec:system}, $h_m^{(t)}\in\mathbb{C}$ is the channel gain from device $m$ to the ES, $\bm{r}^{(t)}\in \mathbb{C}^s$ is the channel output received by the ES, and $\bm{w} \in \mathbb{C}^s$ is an independent additive white Gaussian noise (AWGN) with each element independent and identically distributed (i.i.d.) as $\mathcal{CN}(0,\sigma_w^2)$. During the training process, the power consumption of device $m$ at each round $t$ is constrained by 	\begin{equation}\label{sec:system,ssec:trans,equ:power constraint}
	||\bm{s}_m^{(t)}||^2 \leq P,
\end{equation}
where $P$ is the common power budget of each device and $||\cdot||$ denotes the $l_2$ norm.

The remaining issue is to map the real vectors $\{\bm{g}_{nm}^{(t)}\}_{n=1}^N$ to the complex vector $\bm{s}_m^{(t)}$ at device $m$ in each communication round $t$, and to recover an approximate estimate of $\nabla_n \mathcal{L}_n(\bm{\theta}_{n})$ from $\bm{r}_{(t)}$ for each task $n$ at the ES. As inspired by \cite{amiri_federated_2020,amiri_machine_2020,ma_turbo_2014,seide_1-bit_2014}, we employ analog compressed sensing and error accumulation techniques to combat the effect of channel imperfection. The details of the uplink transceiver design are presented in what follows.

\section{Over-the-Air Federated Multi-Task Learning Framework}\label{sec:system}

\subsection{Transceiver Design of OA-FMTL}

\subsubsection{Transmitter Design}\label{sec:system,ssec:trans}
We process the local gradients of each task $n$ on device $m$ by essentially following the approach in \cite{amiri_federated_2020}. Specifically, at each round $t$, device $m$ adds $\bm{g}_{nm}^{(t)}$ defined in (\ref{model_update}) with the error accumulation term $\bm{\triangle}_{nm}^{(t)}\in \mathbb{R}^{d_n}$ as
\begin{equation}\label{sec:system,ssec:trans,equ:error accumulated}
	\bm{g}_{nm}^{\operatorname{ac}(t)} = \bm{g}_{nm}^{(t)} + \bm{\triangle}_{nm}^{(t)},\forall m\in[M],\forall n\in[N],
\end{equation}
where $\bm{\triangle}_{nm}^{(t)}$ is accumulated in the previous rounds with $\bm{\triangle}_{nm}^{(1)}$ initialized to $\bm{0}$. Then device $m$ sets all the elements of $\bm{g}_{nm}^{\operatorname{ac}(t)}\in\mathbb{R}^{d_n}$ but the $k_n$ elements with the greatest absolute values to zero, defined by
\begin{equation}\label{sec:system,ssec:trans,equ:sparse} 	
	\bm{g}_{nm}^{\operatorname{sp}(t)} = \operatorname{sp}(\bm{g}_{nm}^{\operatorname{ac}(t)}, k_n)\in\mathbb{R}^{d_n},\forall m\in[M],\forall n\in[N].
\end{equation}
$\bm{\triangle}_{nm}^{(t)}$ is updated by
\begin{equation}\label{sec:system,ssec:trans,equ:get error}
	\bm{\triangle}_{nm}^{(t+1)} = \bm{g}_{nm}^{\operatorname{ac}(t)} - \bm{g}_{nm}^{\operatorname{sp}(t)},\forall m\in[M],\forall n\in[N].
\end{equation}
Then $\bm{g}_{nm}^{\operatorname{sp}(t)}$ is compressed into a low-dimensional vector $\bm{g}_{nm}^{\operatorname{cp}(t)} \in \mathbb{R}^{2s}$ by a compression matrix $\bm{A}_n\in \mathbb{R}^{2s\times d_n}$ as
\begin{equation}\label{sec:system,ssec:trans,equ:compress}
	\bm{g}_{nm}^{\operatorname{cp}(t)} = \bm{A}_n\bm{g}_{nm}^{\operatorname{sp}(t)},\forall m\in[M],\forall n\in[N].
\end{equation}
We employ a partial discrete cosine transform (DCT) matrix $\bm{A}_n=\bm{S}_n\bm{F}$ for each task $n$, where the selection matrix $\bm{S}_n\in \mathbb{R}^{2s\times d_n}$ consists of $2s$ randomly selected and reordered rows of the $d_n\times d_n$ identity matrix $\bm{I}_{d_n}$ and the $(m,n)$-th entry of the unitary DCT matrix $\bm{F}\in \mathbb{R}^{{d_n}\times {d_n}}$ is given by $\sqrt{\frac{2}{d_n}}\operatorname{cos}\left(\frac{(m-1)(2n-1)\pi}{2d_n}\right)$ when $m\neq1$, or $\sqrt{\frac{1}{d_n}}$ when $m=1$. It is known that, compared to other choices of the compression matrix such as the i.i.d. Gaussian matrix, the partial DCT matrix has advantages both in performance and complexity \cite{ma_performance_2015}.

We are now ready to describe the design of $\bm{s}_m^{(t)}$. As a distinct feature of the OA-FMTL framework, we propose to superimpose the local gradients of different tasks to support the multiplexing of the $N$ learning tasks. In specific, with (\ref{sec:system,ssec:trans,equ:error accumulated})-(\ref{sec:system,ssec:trans,equ:compress}), each device $m$ constructs
\begin{equation}\label{sec:system,ssec:trans,equ:mapping}
	\bm{x}_{m}^{(t)}=\sum_{n=1}^N K_{nm}\bm{g}_{nm}^{\operatorname{cp}(t)}\in\mathbb{R}^{2s},
\end{equation}
which is then converted into a complex vector $\tilde{\bm{x}}_{m}^{(t)}\in\mathbb{C}^{s}$, defined as
\begin{subequations}
	\begin{equation}
		\operatorname{Re}\{\tilde{\bm{x}}_{m}^{(t)}\}\triangleq \left[x_{m,1}^{(t)},\dots,x_{m,s}^{(t)}\right]^T, 	
	\end{equation} 	
	\begin{equation} 	
		\operatorname{Im}\{\tilde{\bm{x}}_{m}^{(t)}\}\triangleq \left[x_{m,s+1}^{(t)},\dots,x_{m,2s}^{(t)}\right]^T,
	\end{equation}
\end{subequations}
where $x_{m,k}^{(t)}$ is the $k$-th entry of $\bm{x}_{m}^{(t)}$. After that, every device $m$ concurrently sends $\bm{s}_m^{(t)}=\alpha_m^{(t)}\tilde{\bm{x}}_{m}^{(t)}$ to the ES by analog transmission, where $\alpha_m^{(t)}\in\mathbb{C}$ is determined by
\begin{equation}\label{sec:system,ssec:trans,equ:power_allocation} 		
	\alpha_{m}^{(t)}=\left\{ 		
	\begin{aligned} 			&\frac{\gamma^{(t)}}{h_{m}^{(t)}},&\text{if}\left|h_{m}^{(t)}\right|\geq\zeta^{(t)},\\ 			&0,&\text{otherwise}, 		\end{aligned} 		\right.
\end{equation} 	where the power coefficient $\gamma^{(t)}\in\mathbb{R}$ and the threshold $\zeta^{(t)}\in\mathbb{R}$ are set to satisfy the average transmit power constraint (\ref{sec:system,ssec:trans,equ:power constraint}) and inverse $h_{m}^{(t)}$ in (\ref{sec:system,ssec:trans,equ:channel}), respectively. Accordingly, the set of devices scheduled to transmit at the $t$-th round is given by 	
\begin{equation}\label{sec:system,ssec:trans,equ:m_size}
	\mathcal{M}^{(t)}=\left\{m \in[M]:|h_{m}^{(t)}|^2\geq\zeta^{(t)}\right\}. 
\end{equation}

\subsubsection{Receiver Design}
\addtolength{\topmargin}{0.05in}
We now describe the receiver design of the ES. We assume that the ES knows the set $\mathcal{M}^{(t)}$, the size of dataset $K_{nm}$ and the power coefficient $\gamma^{(t)}\in\mathbb{R}$ at each round $t$, in advance of the transmission. With (\ref{sec:system,ssec:trans,equ:compress})-(\ref{sec:system,ssec:trans,equ:m_size}) and appropriate scaling, (\ref{sec:system,ssec:trans,equ:channel}) is rewritten as
\begin{align}\label{sec:system,ssec:trans,equ:channel final}
	\bm{y}^{(t)} &= \sum_{n=1}^N\bm{A}_n\bm{g}_n^{(t)}+ \bm{n}
	\\\nonumber&= 
	\begin{bmatrix}
		\bm{A}_1,\dots,\bm{A}_N
	\end{bmatrix}
	\begin{bmatrix}
		\bm{g}_1^{(t)^T},\dots,\bm{g}_N^{(t)^T}
	\end{bmatrix}^T + \bm{n},
\end{align}
where 
$\bm{y}^{(t)} \triangleq \frac{[\operatorname{Re}\{\bm{r}^{(t)}\}^T,\operatorname{Im}\{\bm{r}^{(t)}\}^T]^T}{\gamma^{(t)}\sum_{m\in\mathcal{M}^{(t)}}K_{nm}}$, 
$\bm{n} \triangleq \frac{[\operatorname{Re}\{\bm{w}\}^T,\operatorname{Im}\{\bm{w}\}^T]^T}{\gamma^{(t)}\sum_{m\in\mathcal{M}^{(t)}}K_{nm}}$ follows $\mathcal{N}(0,\sigma^2)$ with $\sigma \triangleq \frac{\sigma_w}{2\gamma^{(t)}\sum_{m\in\mathcal{M}^{(t)}}K_{nm}}$, and
$\bm{g}_n^{(t)} \triangleq\frac{\sum_{m\in\mathcal{M}^{(t)}}K_{nm}\bm{g}_{nm}^{\operatorname{sp}(t)}}{\sum_{m\in\mathcal{M}^{(t)}}K_{nm}}$ is an approximate sparsified version of $\frac{\sum_{m=1}^{M} K_{nm}\bm{g}_{nm}^{(t)}}{\sum_{m=1}^MK_{nm}}$ in (\ref{model_update}). Then, given $\bm{y}^{(t)}$, the ES reconstructs each $\bm{g}_n^{(t)}$ as $\hat{\bm{g}}_n^{(t)}$ for $\forall n$, in practice, which is subsequently used to update the model parameters via
\begin{equation}\label{sec:system,ssec:multi,equ:paramenter update_new}
	\bm{\theta}_{n}^{(t+1)} = \bm{\theta}_{n}^{(t)} - \eta\hat{\bm{g}}_n^{(t)},\forall n \in [N],
\end{equation}
where $\eta$ is defined below (\ref{update}).

The recovery of $\{\bm{g}_n^{(t)}\}_{n=1}^N$ from $\bm{y}^{(t)}$ in (\ref{sec:system,ssec:trans,equ:channel final}) is a compressed sensing problem with the compression matrix $[\bm{A}_1,\dots,\bm{A}_N]$ composed of $N$ partial DCT matrices. Since the compression matrix is partial orthogonal, we propose to follow the idea of Turbo-CS in \cite{ma_turbo_2014} to efficiently solve the compressed sensing problem. As each $\bm{g}_n^{(t)}$ is the gradient for a different task $n$, $\{\bm{g}_n^{(t)}\}_{n=1}^N$ generally have different prior distributions. We assume that the entries of $\bm{g}_n^{(t)}$ are independently drawn from a Bernoulli Gaussian distribution:
\begin{equation}\label{sec:system,ssec:recei,equ:Bernoulli Gaussian}
	g_{n,k}^{(t)} \sim\left\{\begin{array}{ll}
		0, & \text { probability }=1-\lambda_n^{(t)}, \\
		\mathcal{N}\left(0, v_n^{(t)}\right), & \text { probability }=\lambda_n^{(t)},
	\end{array}\right.
\end{equation}
where $g_{n,k}^{(t)}$ is the $k$-th element of $\bm{g}_n^{(t)}$, $\lambda_n^{(t)}$ is the sparsity of $\bm{g}_n^{(t)}$, and $v_n^{(t)}$ is the variance of the nonzero elements in $\bm{g}_n^{(t)}$. The above parameters in the prior distribution is estimated by the Expectation-Maximization algorithm \cite{vila_expectation-maximization_2013}. With the above prior model, we modify the Turbo-CS algorithm accordingly to accommodate the concurrent model aggregation of the $N$ tasks as follows.

\begin{figure}[h] 	\centering 	\includegraphics[width=1\linewidth, height=0.12\textheight]{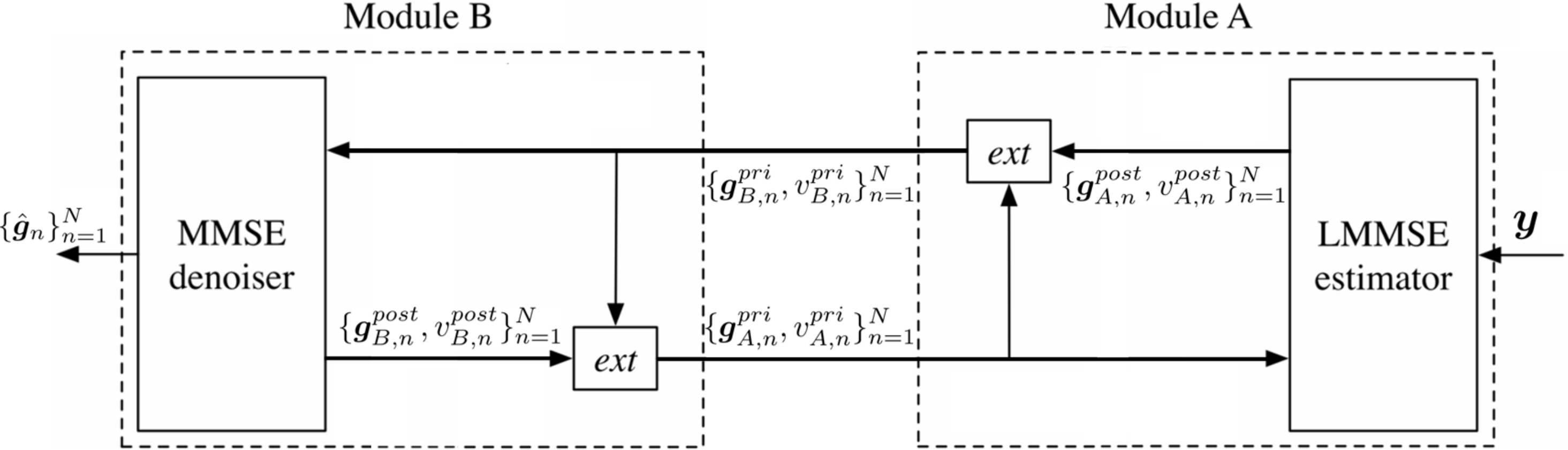} 	\caption{An illustration of the M-Turbo-CS algorithm.} 	\label{sec:system,ssec:recei,fig:Turbo-CS} \end{figure} 
As shown in
Fig. \ref{sec:system,ssec:recei,fig:Turbo-CS}, modified Turbo-CS (M-Turbo-CS) iterates between two modules where module A handles the linear constraint in (\ref{sec:system,ssec:trans,equ:channel final}), and module B denoises the output from module A by exploiting the gradient sparsity in (\ref{sec:system,ssec:recei,equ:Bernoulli Gaussian}). Besides, the iterative process of M-Turbo-CS is carried out at every round $t$ and we drop out the round index $t$ in the following for brevity. At each turbo iteration, given the prior mean $\bm{g}^{pri}_{A,n}\in\mathbb{R}^{d_n}$ and the variance $v^{pri}_{A,n}\in\mathbb{R}$ from module B as well as the observed vector $\bm{y}$ in (\ref{sec:system,ssec:trans,equ:channel final}), the posterior mean $\bm{g}^{post}_{A,n}\in\mathbb{R}^{d_n}$ and the variance $v^{post}_{A,n}\in\mathbb{R}$ of $\bm{g}_n$ are given by
\begin{small}
	\begin{subequations}\label{sec:system,ssec:recei,equ:multi task turbo A}	
		\begin{equation}\label{sec:system,ssec:recei,equ:multi task turbo xA}	
			\bm{g}_{A,n}^{post}=\bm{g}_{A,n}^{pri}+\frac{v_{A,n}^{pri}\bm{A}_n^T(\bm{y}-\sum_{k=1}^N\bm{A}_k\bm{g}^{pri}_{A,k})}{\sum_{k=1}^Nv_{A,k}^{pri}+\sigma^2},\forall n\in[N], 
		\end{equation}
		\vspace{-1.5mm}
		\begin{equation}\label{sec:system,ssec:recei,equ:multi task turbo vA}
			v_{A,n}^{post} = v_{A,n}^{pri} - \frac{s}{d_n}\frac{{v_{A,n}^{pri}}^2}{\sum_{k=1}^N v_{A,k}^{pri}+\sigma^2},\forall n\in[N].	 	
		\end{equation}
	\end{subequations}
\end{small}From (\ref{sec:system,ssec:recei,equ:multi task turbo A}), the prior mean $\bm{g}^{pri}_{B,n}\in\mathbb{R}^{d_n}$ and the variance $v^{pri}_{B,n}\in\mathbb{R}$ of the MMSE denoiser are the extrinsic mean and variance from module A, given by 
\begin{subequations}\label{sec:system,ssec:recei,equ:multi task turbo B_PRI}
	\begin{equation}\label{sec:system,ssec:recei,equ:multi task turbo xB_PRI}
		\bm{g}^{pri}_{B,n} = v^{pri}_{B,n}
		\begin{pmatrix}
			\frac{\bm{g}^{post}_{A,n}}{v^{post}_{A,n}}-\frac{\bm{g}^{pri}_{A,n}}{v^{pri}_{A,n}}
		\end{pmatrix},\forall n\in[N],
	\end{equation}
	\vspace{-3mm}
	\begin{equation}\label{sec:system,ssec:recei,equ:multi task turbo vB_PRI}
		v^{pri}_{B,n} =
		\begin{pmatrix}
			\frac{1}{v^{post}_{A,n}}-\frac{1}{v^{pri}_{A,n}}
		\end{pmatrix}^{-1},\forall n\in[N],
	\end{equation}	
\end{subequations}
respectively. Following \cite{ma_turbo_2014}, each $\bm{g}^{pri}_{B,n}$ is modeled as an observation of $\bm{g}_{n}$ corrupted by additive noise $\bm{n}_{n}$: 
\begin{equation} 		
	\bm{g}_{B,n}^{pri} = \bm{g}_{n} + \bm{n}_{n}, 	
\end{equation}
where $\bm{n}_{n}\sim\mathcal{N}(0,v_{B,n}^{pri})$ is independent of $\bm{g}_{n}$. The posterior mean $\bm{g}^{post}_{B,n}\in\mathbb{R}^{d_n}$ and the variance $v^{post}_{B,n}\in\mathbb{R}$ of the MMSE denoiser are given by
\begin{subequations}\label{sec:system,ssec:recei,equ:multi task turbo B}
	\begin{equation}\label{sec:system,ssec:recei,equ:multi task turbo xB} 		
		\bm{g}_{B,n}^{post} = \mathbb{E}[\bm{g}_{n} \mid \bm{g}_{B,n}^{pri}],\forall n\in[N],
	\end{equation}
	\vspace{-5mm}
	\begin{equation}\label{sec:system,ssec:recei,equ:multi task turbo vB}
		v_{B,n}^{post} = \sum_{k=1}^d
		\operatorname{var}[g_{n,k} \mid g_{B,n,k}^{pri}],\forall n\in[N], 
	\end{equation}
\end{subequations}
where the expectation $\mathbb{E}$ is with respect to $\bm{g}_{n}$, $\operatorname{var}[a|b]=\mathbb{E}[|a-\mathbb{E}[a|b]|^2|b]$ and $g_{n,k}$ or $g_{B,n,k}^{pri}$ is the $k$-th element of $\bm{g}_{n}$ or $\bm{g}_{B,n}^{pri}$, respectively. 
The prior mean $\bm{x}^{pri}_{A,n}\in\mathbb{R}^{d_n}$ and the variance $v^{pri}_{A,n}\in\mathbb{R}$ of module A are updated by
\begin{subequations}\label{sec:system,ssec:recei,equ:multi task turbo A_pri}
	\begin{equation}\label{sec:system,ssec:recei,equ:multi task turbo xA_pri} 		
		\bm{g}^{pri}_{A,n} =v^{pri}_{A,n}
		\begin{pmatrix}
			\frac{\bm{g}^{post}_{B,n}}{v^{post}_{B,n}}-\frac{\bm{g}^{pri}_{B,n}}{v^{pri}_{B,n}}
		\end{pmatrix},\forall n\in[N],	
	\end{equation}
	\vspace{-3mm}
	\begin{equation}\label{sec:system,ssec:recei,equ:multi task turbo vA_pri}	 	 		
		v^{pri}_{A,n} =
		\begin{pmatrix}
			\frac{1}{v^{post}_{B,n}}-\frac{1}{v^{pri}_{B,n}}
		\end{pmatrix}^{-1},\forall n\in[N].  	
	\end{equation}
\end{subequations}
\addtolength{\topmargin}{0.05in}
To sum up, given the initialization values $\bm{g}_{A,n}^{pri}=\bm{0}$ and $v_{A,n}^{pri}=v_n^{ini}$ for $\forall n\in[N]$, (\ref{sec:system,ssec:recei,equ:multi task turbo A})-(\ref{sec:system,ssec:recei,equ:multi task turbo A_pri}) iterate until some termination criterion is met, and $\bm{g}_{B,n}^{post}$ is output as $\hat{\bm{g}}_n$ for the model update in (\ref{sec:system,ssec:multi,equ:paramenter update_new}), for $\forall n\in[N]$. Here, $v_n^{ini}$ can be set to $v_n$ in (\ref{sec:system,ssec:recei,equ:Bernoulli Gaussian}). Note that in practice, $\{v_n\}_{n=1}^N$ may be difficult to determine in prior. However, empirically, the algorithm is not very sensitive to the initial variances, and thus we approximately set $v_1^{ini}=\dots=v_N^{ini}=\frac{||\bm{y}||^2}{Ns}$. Compared with the original Turbo-CS algorithm in \cite{ma_turbo_2014}, the main difference is that each subvector $\bm{g}_n$ in (\ref{sec:system,ssec:trans,equ:channel final}) has its individual prior distribution as in (\ref{sec:system,ssec:recei,equ:Bernoulli Gaussian}). 
The above process is summarized in Algorithm \ref{sec:system,ssec:alogr,alg:algorithm}.

\subsection{Performance Analysis of M-Turbo-CS}\label{Performance Analysis of Turbo-CS}
 Similarly to \cite{ma_turbo_2014}, we track the state of each $\bm{g}_n$ with its individual MSE. Combining (\ref{sec:system,ssec:recei,equ:multi task turbo vA}), (\ref{sec:system,ssec:recei,equ:multi task turbo vB_PRI}), (\ref{sec:system,ssec:recei,equ:multi task turbo vB}) and (\ref{sec:system,ssec:recei,equ:multi task turbo vA_pri}), the state evolution of M-Turbo-CS is given by 
\begin{subequations}\label{sec:conver,equ:state evolution} 		
	\begin{equation}\label{sec:conver,equ:state evolution 1} 		
		v_{B,n}^{pri} = \frac{d_n}{s}\begin{pmatrix}
			\sum_{k=1}^Nv_{A,k}^{pri}+\sigma^2
		\end{pmatrix}-v_{A,n}^{pri}, \forall n\in[N],	
	\end{equation}
	\vspace{-3mm}
	\begin{equation}\label{sec:conver,equ:state evolution 2} 	
		\frac{1}{v_{A,n}^{pri}}=\frac{1}{ mmse_n(1/v_{B,n}^{pri})} - \frac{1}{v_{B,n}^{pri}},\forall n\in[N],
	\end{equation}
\end{subequations}
with  
$ 		
mmse_n(1/v_{B,n}^{pri}) \equiv \mathbb{E}\left[|\bm{g}_n-\mathbb{E}[\bm{g}_n \mid \bm{g}_n+\bm{n}_n]|^{2}\right].  	
$ The fixed point of (\ref{sec:conver,equ:state evolution}), denoted by $\{v_n^\star\}_{n=1}^N$, tracks the normalized output MSEs of the M-Turbo-CS  algorithm, where $v_n^\star$ is the fixed-point MSE of $\bm{g}_n$ for $\forall n\in N$. The fixed point $\{v_n^\star\}_{n=1}^N$ gives an analytical characterization of the communication error after turbo recovery. This error bound will be used in the next subsection for convergence analysis of the overall OA-FMTL. Moreover, we will numerically show that the state evolution in (\ref{sec:conver,equ:state evolution}) agrees well with simulation, and that M-Turbo-CS is able to efficiently suppress inter-task interference.

\begin{algorithm}[h]
	\caption{ OA-FMTL alogrithm.}\label{sec:system,ssec:alogr,alg:algorithm}  		
	\begin{algorithmic}[1] 
		\STATE \textbf{Initialize} $\bm{\triangle}_{nm}^{(1)}=\bm{0},\forall n\in[N],m\in[M]$
		\FOR{$t = 1,2,\dots$}
		\STATE \textbf{Each device $m$ does in parallel:}	
		
		\STATE Compute $\{\bm{g}_{nm}^{(t)}\}_{n=1}^N$ with $\{D_{nm}\}_{n=1}^N$ and $\{\bm{\theta}_{n}^{(t)}\}_{n=1}^N$
		\STATE Compute $\{\bm{s}_m^{(t)}\}_{n=1}^N$ via (\ref{sec:system,ssec:trans,equ:error accumulated})-(\ref{sec:system,ssec:trans,equ:sparse}) and (\ref{sec:system,ssec:trans,equ:compress})-(\ref{sec:system,ssec:trans,equ:m_size})
		\STATE Compute $\{\bm{\triangle}_{nm}^{(t+1)}\}_{n=1}^N$ via (\ref{sec:system,ssec:trans,equ:get error})
		
		\STATE Send $\bm{s}_m^{(t)}$ to the ES
		synchronously with other devices
		\STATE \textbf{ES does:}
		\STATE Receive $\bm{r}^{(t)}$ via (\ref{sec:system,ssec:trans,equ:channel}) and compute $\bm{y}^{(t)}$ via (\ref{sec:system,ssec:trans,equ:channel final})
		\STATE \textbf{Initialize} $\bm{g}_{A,n}^{pri}=\bm{0}$, $v_{A,n}^{pri}=v_n^{ini}, \forall n\in[N]$
		\REPEAT 
		\STATE Update $\{\bm{g}_{B,n}^{post}\}_{n=1}^N$, via (\ref{sec:system,ssec:recei,equ:multi task turbo A})-(\ref{sec:system,ssec:recei,equ:multi task turbo A_pri})
		\UNTIL{convergence}
		\STATE $\hat{\bm{g}}_n^{(t)} = \bm{g}_{B,n}^{post}, \forall n\in[N]$
		\STATE $\bm{\theta}_{n}^{(t+1)} = \bm{\theta}_{n}^{(t)} -\eta \hat{\bm{g}}_n^{(t)}, \forall n\in[N]$
		\STATE Broadcast $\{\bm{\theta}_n^{(t+1)}\}_{n=1}^N$ to all the devices
		\ENDFOR
	\end{algorithmic} 	
\end{algorithm}

\subsection{Convergence Analysis of OA-FMTL}\label{sec:conver,ssec:state evolution}
\newtheorem{assumption}{Assumption}
\newtheorem{lemma}{Lemma}
\newenvironment{sproof}{{\indent \indent \it Sketch of proof:}}{\hfill$\blacksquare$\par}
\newenvironment{proof}{{\indent \indent \it  Proof:}}{\hfill$\blacksquare$\par}
\newtheorem{thm}{Theorem}
\newtheorem{remark}{Remark}
\newtheorem{coro}{Corollary}


We now analyze the performance of the OA-FMTL framework. 
With (\ref{update}) and (\ref{sec:system,ssec:multi,equ:paramenter update_new}), we analyze the bound of model updating error $\bm{e}^{(t)}\in\mathbb{R}^{d_n}$ at the $t$-th round as
\begin{subequations}\label{sec:system,ssec:conve,equ:error_all}
	\begin{align}\label{sec:system,ssec:conve,equ:error}		
		||\bm{e}^{(t)}||^2 &=\left|\left|\nabla \mathcal{L}(\bm{\theta}^{(t)})-
		\begin{bmatrix} 				\hat{\bm{g}}_1^{(t)^T},\dots,\hat{\bm{g}}_N^{(t)^T} 			\end{bmatrix}^T\right|\right|^2\tag{25a}\\ 			&=\sum_{n=1}^N||\nabla_n \mathcal{L}_n(\bm{\theta}_{n})-\hat{\bm{g}}_n^{(t)}||^2\tag{25b} \\&=\sum_{n=1}^N||\bm{e}_n^{(t)}||^2,\tag{25c}	
	\end{align} 
\end{subequations}
with the error $\bm{e}_{n}^{(t)}\in\mathbb{R}^{d_n}$ from task $n$ characterized by
\begin{subequations}\label{sec:system,ssec:conve,equ:errorn_all}
	\begin{align}\label{sec:system,ssec:conve,equ:error_n}  		
		\bm{e}_{n}^{(t)} = & \ \nabla_n \mathcal{L}_n(\bm{\theta}_{n})-\hat{\bm{g}}_n^{(t)}\tag{26a}\\ 		
		\nonumber =& \ \underbrace{\frac{\sum_{m=1}^{M} K_{nm}\bm{g}_{nm}^{(t)}}{\sum_{m=1}^MK_{nm}}-\frac{\sum_{m=1}^{M} K_{nm}\bm{g}_{nm}^{\operatorname{sp}(t)}}{\sum_{m=1}^MK_{nm}}}_{\text{Sparsification error}}\\
		\nonumber&\ +\underbrace{\frac{\sum_{m=1}^{M} K_{nm}\bm{g}_{nm}^{\operatorname{sp}(t)}}{\sum_{m=1}^MK_{nm}}-\frac{\sum_{m\in\mathcal{M}^{(t)}}^{(t)} K_{nm}\bm{g}_{nm}^{\operatorname{sp}(t)}}{\sum_{m\in\mathcal{M}^{(t)}}K_{nm}}}_{\text{User selection error}}\\ 		&\ +\underbrace{\frac{\sum_{m\in\mathcal{M}^{(t)}}^{(t)} K_{nm}\bm{g}_{nm}^{\operatorname{sp}(t)}}{\sum_{m\in\mathcal{M}^{(t)}}K_{nm}} 		-\hat{\bm{g}}_n^{(t)}}_{\text{Estimation error from M-Turbo-CS}}\tag{26b}\\=& \ \bm{e}_{n,1}^{(t)}+\bm{e}_{n,2}^{(t)}+\bm{e}_{n,3}^{(t)}, \tag{26c}	
	\end{align} 
\end{subequations}		
where $\bm{e}_{n,1}^{(t)}\in\mathbb{R}^d$ denotes the sparsification error caused by the step in (\ref{sec:system,ssec:trans,equ:sparse}), $\bm{e}_{n,2}^{(t)}\in\mathbb{R}^d$ denotes the user selection error caused by the fading channel with (\ref{sec:system,ssec:trans,equ:power_allocation}), and $\bm{e}_{n,3}^{(t)}\in\mathbb{R}^d$ denotes the estimation error caused by the imperfect recovery of M-Turbo-CS. With (\ref{sec:system,ssec:conve,equ:errorn_all}), we bound $||\bm{e}_n^{(t)}||^2$ as
\begin{align}\label{expectation error}
	||\bm{e}_n^{(t)}||^2 \leq  3(||\bm{e}_{n,1}^{(t)}||^2+||\bm{e}_{n,2}^{(t)}||^2+||\bm{e}_{n,3}^{(t)}||^2),
\end{align}
by using the triangle inequality and the inequality of arithmetic means. The analysis of $||\bm{e}_{n,1}^{(t)}||^2$ defined in (\ref{error:e1}) and $||\bm{e}_{n,2}^{(t)}||^2$ defined in (\ref{error:e2}) basically follows the process in \cite{amiri_machine_2020} and \cite{liu_reconfigurable_2021}, respectively, and is omitted for simplicity. Besides, from the performance analysis of M-Turbo-CS in Section \ref{Performance Analysis of Turbo-CS}, M-Turbo-CS converges to the fixed point $\{{v_n^\star}^{(t)}\}_{n=1}^N$ at the $t$-th round, for $\forall n\in[N]$. Thus, we have $||\bm{e}_{n,3}^{(t)}||^2=d_n{v_n^{\star}}^{(t)}$.

To proceed, following the convention in stochastic optimization \cite{friedlander_erratum_2011} to ensure an upper bound on the loss $\mathcal{L}_n(\cdot)$ for each task $n$, we make some assumptions below.
\begin{assumption}\label{sec:conver,asu:assumption 1}
	$\mathcal{L}_n(\cdot)$ is strongly convex with some (positive) parameter $\Omega_n$. That is, $\mathcal{L}_n(\bm{y}) \geq \mathcal{L}_n(\bm{x})+(\bm{y}-\bm{x})^{T} \nabla_n \mathcal{L}_n(\bm{x})+\frac{\Omega_n}{2} \| \bm{y}-\bm{x} \|^{2}, \forall \bm{x}, \bm{y} \in \mathbb{R}^{d_n},\forall n\in[N]$.
\end{assumption}  	 	\begin{assumption}\label{sec:conver,asu:assumption 2}
	The gradient $\nabla_n \mathcal{L}_n(\cdot)$ is Lipschitz continuous with some (positive) parameter $L$. That is, $\|\nabla_n \mathcal{L}_n(\bm{x})-\nabla_n \mathcal{L}_n(\bm{y})\| \leq L_n\|\bm{x}-\bm{y}\|, \forall \bm{x}, \bm{y} \in \mathbb{R}^{d_n},\forall n\in[N]$.
\end{assumption}  	 	\begin{assumption}\label{sec:conver,asu:assumption 3} 			
	$\mathcal{L}_n(\cdot)$ is twice-continuously differentiable, for $\forall n\in[N]$.
\end{assumption}  	 	\begin{assumption}\label{sec:conver,asu:assumption 4}
	The gradient with respect to any training sample, denoted by $\nabla_n l_n(\bm{\theta}_{n}; \cdot)$, is upper bounded at $\bm{\theta}_{n}$ as 	$$\nonumber\left\|\nabla_n l_n(\bm{\theta}_{n},\bm{u}_{nmk})\right\|^{2} \leq \beta_{n,1}+ \beta_{n,2} \left\|\nabla_n \mathcal{L}_n\left(\bm{\theta}_n\right)\right\|^{2},\forall n\in[N]$$ for some constants $\beta_{n,1} \geq 0$ and $\beta_{n,2}>0$.
\end{assumption}

Assumptions \ref{sec:conver,asu:assumption 1}-\ref{sec:conver,asu:assumption 4} lead to an upper bound on the loss function $\mathcal{L}_n(\bm{\theta}_n^{(t+1)})$ with respect to the recursion (\ref{sec:system,ssec:multi,equ:paramenter update_new}) with an arbitrary choice of the learning rate $\eta$. The details are given in the following lemma.
\begin{lemma}\label{lemma:update} 
	Let $\mathcal{L}_n(\cdot)$ satisfy Assumptions \ref{sec:conver,asu:assumption 1}-\ref{sec:conver,asu:assumption 4}. At the $t$-th training round, with $L_n =1/\eta$, we have
	\begin{align}\label{equ:loss_update} 
		\mathcal{L}_n(\bm{\theta}_n^{(t+1)}) \leq \mathcal{L}_n(\bm{\theta}_n^{(t)})-\frac{1}{2L_n}\|\nabla_n \mathcal{L}_n(\bm{\theta}_n^{(t)})\|^{2}+\frac{1}{2L_n} \|\bm{e}_n^{(t)}\|^{2},
	\end{align}
	where the Lipschitz constant $L_n$ is defined in Assumption \ref{sec:conver,asu:assumption 2}.
\end{lemma}
\begin{proof}
	See \cite[Lemma 2.1]{friedlander_erratum_2011}.
\end{proof}

We are now ready to derive an upper bound of the difference between the training loss and the optimal loss, i.e., $\mathcal{L}(\bm{\theta}^{(t+1)})-\mathcal{L}(\bm{\theta}^{(\star)})$. 
\begin{thm}\label{thm:1}
	With Assumptions \ref{sec:conver,asu:assumption 1}-\ref{sec:conver,asu:assumption 4},
	\begin{align}\label{equ:final_convergence}
		\nonumber\mathcal{L}(\bm{\theta}^{(t+1)})&-\mathcal{L}(\bm{\theta}^{(\star)})\leq
		\left(\mathcal{L}(\bm{\theta}^{(1)})-\mathcal{L}(\bm{\theta}^{(\star)})\right)\prod_{t'=1}^t\max_n\Upsilon_n^{(t')}\\ & +\sum_{n=1}^{N}\sum_{t''=1}^{t}C_n^{(t'')}\prod_{t'=t''}^{t-1}\Upsilon_n^{(t'+1)},
	\end{align}
	where operation $\prod_a^b(\cdot)=1$ when $a>b$, $\mathcal{L}(\cdot)$ is the total empirical loss function defined in (\ref{sec:system,ssec:multi,equ:optimum function}), $\bm{\theta}^{(1)}$ is the initial system model parameter, and the functions $\Upsilon_n^{(t)}$, $C_n^{(t)}$ for each task $n$ are defined as
	\begin{subequations}\label{define:psi,upsilon_C}
		\begin{align}
			\Upsilon_n^{(t)}\triangleq1-\frac{\Omega_n}{L_n}+\frac{2\Omega_n\beta_{n,2}\Psi_n^{(t)}}{L_n},\\
			C_n^{(t)}\triangleq\frac{\beta_{n,1}}{L_n}\Psi_n^{(t)}+\frac{3d_n{v_n^\star}^{(t)}}{2L_n},
		\end{align}
	\end{subequations}
	with the function $\Psi_n^{(t)}$ for each task $n$ defined by
	\begin{align}\label{define:psi}
		&\nonumber\Psi_n^{(t)}\triangleq\\&\frac{3}{2}\left(\left(\frac{2r_n-r_n^{t}-r_n^{t+1}}{1-r_n}\right)^2+\left(2-2\frac{\sum_{m \in \mathcal{M}^{(t)}} K_{nm}}{\sum_{m=1}^MK_{nm}}\right)^{2}\right),
	\end{align}
	where $\mathcal{M}^{(t)}$ is defined in (\ref{sec:system,ssec:trans,equ:m_size}), $r_n^t$ denotes the $t$-th power of $r_n$, and $r_n=\sqrt{(d_n-k_n)/d_n}<1$ with $k_n$ is defined above (\ref{sec:system,ssec:trans,equ:sparse}). In the above, the parameters ${L_n, \Omega_n, \beta_{n,1}, \beta_{n,2}}$ are defined in Assumptions \ref{sec:conver,asu:assumption 1}-\ref{sec:conver,asu:assumption 4}.
\end{thm}
\begin{proof}
	See Appendix \ref{proof:Theorem 1}.
\end{proof}

From Theorem \ref{thm:1}, we see that $\mathcal{L}(\bm{\theta}^{(t+1)})-\mathcal{L}(\bm{\theta}^{(\star)})$ denotes the difference between the training loss and the optimal loss at $t$-th round, which is upper bounded by the right side of the inequality in (\ref{equ:final_convergence}). In particular, $\mathcal{L}(\bm{\theta}^{(1)})-\mathcal{L}(\bm{\theta}^{(\star)})$ in this bound denotes the difference between the initialization loss and the optimal loss, and the second term of this bound is associated with the system error, including the sparsification error, the M-Turbo-CS estimation error, and the user selection error. We note that $\mathcal{L}(\bm{\theta}^{(t+1)})$ converges with speed $\Upsilon^{(t)}=\max_n\Upsilon_n^{(t')}$ when $\Upsilon^{(t)}<1$. This condition holds when we choose $\Upsilon_n^{(t)}<1$ for $\forall n\in[N]$ at each communication round $t$. Empirically, we find that the proposed OA-FMTL scheme always converges with appropriately chosen system parameters. Moreover, we emphasize that the upper bound in (\ref{equ:final_convergence}) gives a performance metric of OA-FMTL and thus can be potentially used for system performance optimization. Due to space limitation, we leave more detailed discussions on system optimization to the extended version of this paper.

\section{Experimental Results}
In this section, we validate our proposed OA-FMTL scheme with experiments. Specifically, we consider federated image classification tasks on the MNIST and the Fashion-MNIST datasets, i.e., $N=2$, among $M=20$ local devices and an ES. For each task, we train a convolutional neural network with two convolution layers and two fully connected layers. Since user selection is not the focus of this paper, the channel gain $\{h_m^{(t)}\}_{m=1}^M$ and the threshold $\zeta^{(t)}$ at each communication round defined in (\ref{sec:system,ssec:trans,equ:power_allocation}) are appropriately chosen to ensure that $\mathcal{M}^{(t)}$ consists of $M=20$ devices during the whole training process. Besides, we set $P=0.1, d_1=d_2=10920, k_1=k_2=0.1d, \gamma^{(t)}=1000, \eta=0.1$ for the following experiments. For comparison, we include the following two baseline schemes:
\begin{itemize}
	\item Scheme I: Time division multiplexing (TDM) among the tasks is applied, i.e., each task is assigned with an orthogonal time slot to avoid inter-task interference. OA-FL \cite{amiri_machine_2020} is applied in transmission, and Turbo-CS \cite{ma_turbo_2014} is applied to recover the model aggregation at the ES.
	\item Scheme II: The proposed OA-FMTL framework is applied to jointly transmit the model parameters of all the tasks concurrently, and the Turbo-CS algorithm is used to individually recover the model aggregation of each task without considering the existence of inter-task interference.
\end{itemize}

Fig. \ref{fig:turboperformance} shows the numerical results of signal recovery at $t=90$ round. We see that the simulation results agree well with their corresponding evolution results, and that there is no state evolution result of Scheme II. Compared with Scheme II, we also notice that our proposed scheme performs better in overcoming inter-task interference. Besides, due to avoiding inter-task interference through TDM, Scheme I outperforms others in terms of converged MSE. However, subsequent experiment will show that the learning performance of our proposed scheme is comparable to that of Scheme I.

\begin{figure}[h]
\centering
\includegraphics[width=1\linewidth, height=0.2\textheight]{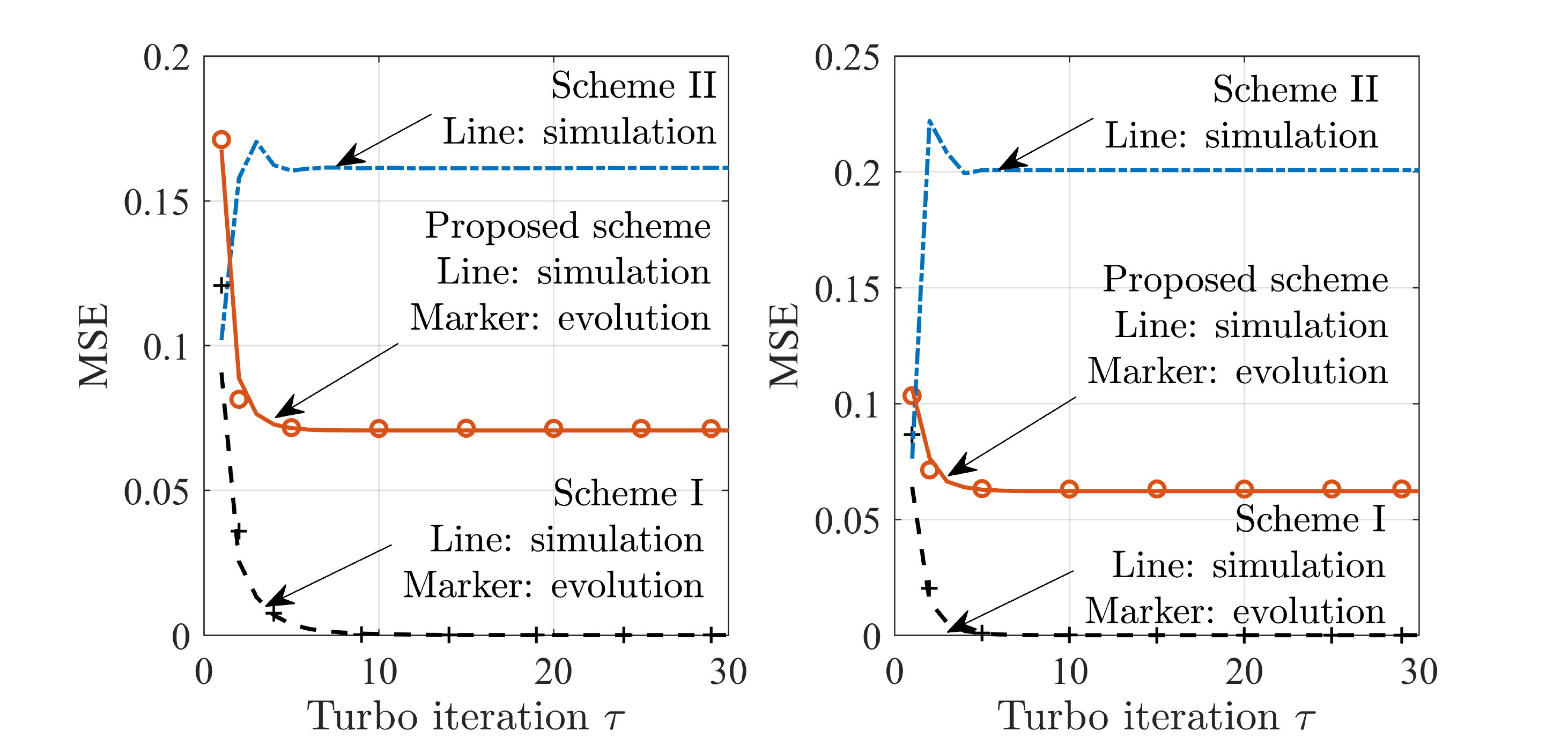}
\caption{The MSE performance on two tasks with $\sigma_w^2=0.1, 2s/d_1=2s/d_2=3/4, K_{nm}=2500, \forall n\in[N], \forall m\in[M],$ at the communication round $t=90$. (a) For MNIST task, $\lambda_1^{(90)}=0.5515, v_1^{(90)}=0.2175$. (b) For Fashion-MNIST task, $\lambda_2^{(90)}=0.5230, v_2^{(90)}=0.1281$.}
\label{fig:turboperformance}
\end{figure}

In Fig. \ref{fig:accuracy}, we measure the performance of each task in terms of test accuracy versus communication round $t$. We observe that the test accuracy of our proposed scheme is close to that of Scheme I, and both schemes converge to an accuracy of 0.9 on the MNIST task as well as to an accuracy of 0.72 on the Fashion-MNIST task. We also note that the test accuracy of our proposed scheme is better than that of Scheme II, and is only about 2\% lower than that of the ideal error-free bound, which demonstrates the excellent interference suppression capability of OA-FMTL with M-Turbo-CS.

\begin{figure}[h]
	\centering
	\includegraphics[width=0.95\linewidth, height=0.25\textheight]{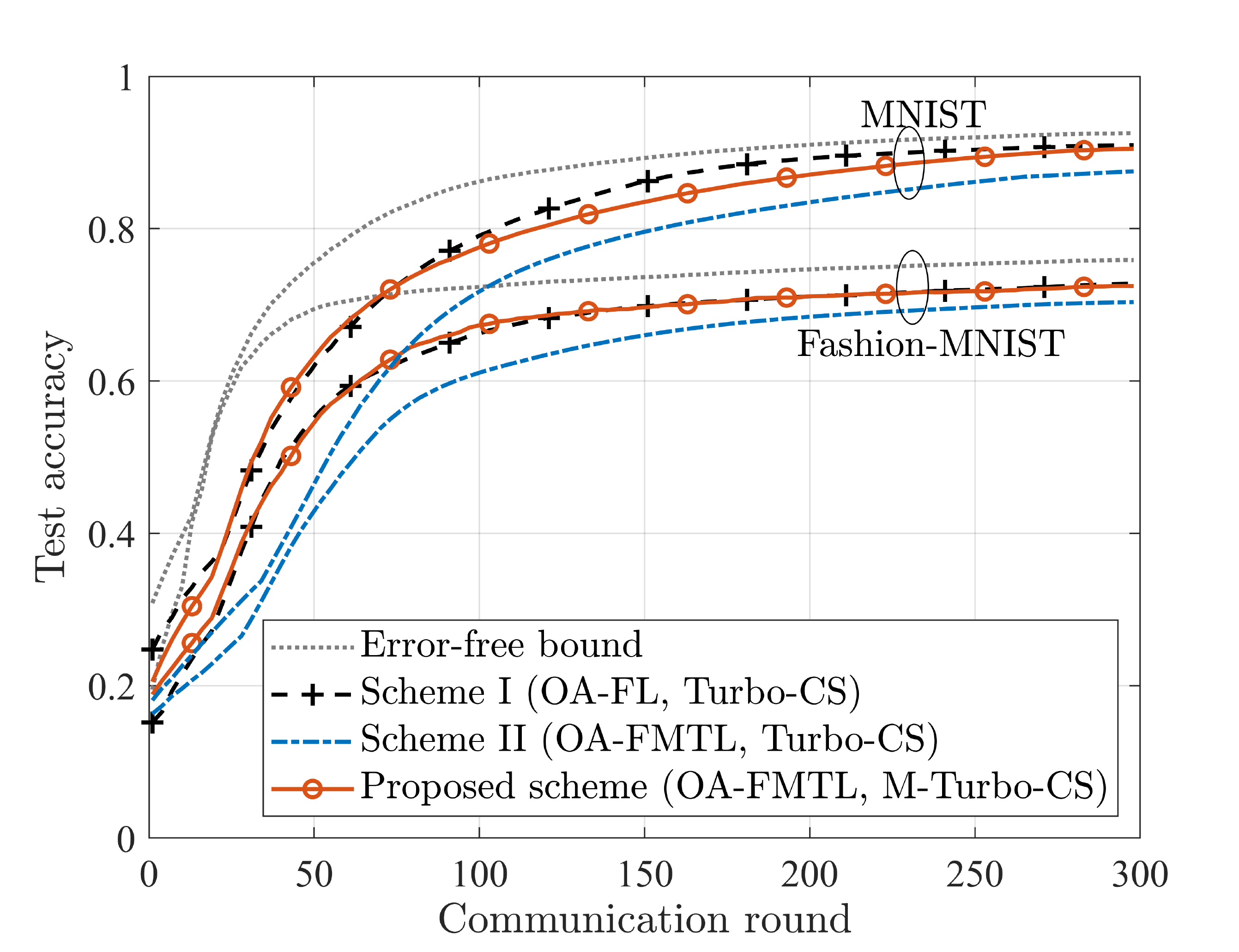}
	\caption{The test accuracies of the two tasks with $2s/d_1=2s/d_2=3/4, \sigma_w^2=0.1, K_{nm}=2500, \forall n\in[N], \forall m\in[M]$.}
	\label{fig:accuracy}
\end{figure}

For further comparison on the number of channel uses, we define $\xi_n^{max}$ as the maximum test accuracy of each task $n$, and define $t^\star(\xi)$ as the total required rounds of communications for every task $n$ to reach its target accuracy $\xi\cdot\xi_n^{max}$, where $\xi$ is called the relative target accuracy. Thus, $t^\star(\xi)$ for Scheme II and our proposed scheme is given by
\begin{equation}\label{sec:system,ssec:perform,equ:MOFL perform 1}		 	 		
	t^\star(\xi) =max\{t_1^\star(\xi),\dots,t_N^\star(\xi)\},
\end{equation}
where $t_n^\star(\xi)$ is the required communication rounds of task $n$ to reach its target accuracy $\xi\cdot\xi_n^{max}$. Owing to the TDM technology, $t^\star(\xi)$ of Scheme I is given by
\begin{equation}		 	
	t^\star(\xi) =\sum_{n=0}^N t_n^\star(\xi). 
\end{equation}
Fig. \ref{fig:performance} depicts the total required communication rounds $t^\star$ versus relative target accuracy $\xi$. We see that our proposed OA-FMTL scheme significantly outperforms the other two baseline schemes, and that the total required communication rounds of our proposed scheme to complete the $N=2$ tasks are only half of that of Scheme I at any value of $\xi$. In addition, we note that Scheme II also requires fewer communication rounds than Scheme I, which demonstrates the advantage of non-orthogonal transmission.
\begin{figure}[h]
	\centering
	\includegraphics[width=0.95\linewidth, height=0.23\textheight]{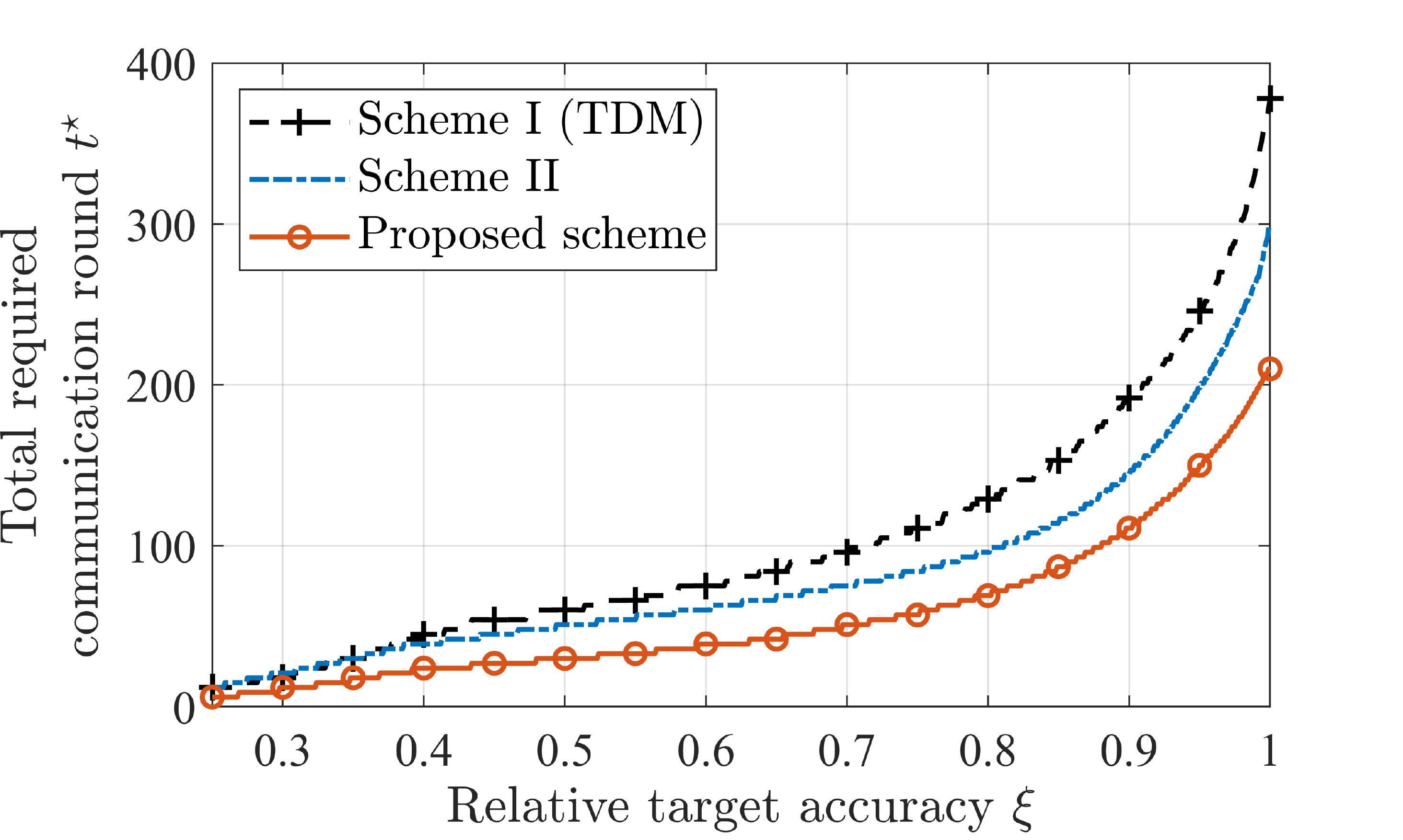}
	\caption{The required communication rounds $t^\star$ of interference-free, with $2s/d_1=2s/d_2=3/4, \xi_1^{max}=0.90, \xi_2^{max}=0.72, \sigma_w^2=0.1, K_{nm}=2500,\forall n\in[N],\forall m\in[M]$.}
	\label{fig:performance}
\end{figure}

\section{Conclusion}
We developed an OA-FMTL framework with over-the-air computation to support multiple learning tasks over an non-orthogonal uplink channel. Furthermore, we modified the original Turbo-CS algorithm in the compressed sensing context to reconstruct the sparsified model aggregation updates at ES. Both the convergence analysis and experimental results showed that our proposed OA-FMTL framework is not sensitive to the inter-task interference, thereby achieving significant reduction in the total number of channel uses with only slight learning performance degradation.

\appendices

\section{Proof of Theorem 1}\label{proof:Theorem 1}

First, following \cite[Appendix A]{amiri_machine_2020}, we bound $||\bm{e}_{n,1}^{(t)}||^2$ as
\begin{align}\label{error:e1}	 
	\left\|\bm{e}_{n,1}^{(t)}\right\|^2\leq&\left (\frac{2r_n-r_n^{t}-r_n^{t+1}}{1-r_n}\right)^2\\
	&\nonumber\times\left(\beta_{n,1}+\beta_{n,2}\left\|\nabla_n \mathcal{L}_n\left(\bm{\theta}_n^{(t)}\right)\right\|^{2}\right),		 
\end{align}
where $r_n=\sqrt{(d_n-k_n)/d_n}<1$ with $k_n$ defined above (\ref{sec:system,ssec:trans,equ:sparse}), $r_n^t$ denotes the $t$-th power of $r_n$, and $\beta_{n,1}$ as well as $\beta_{n,2}$ are some constants defined in Assumption \ref{sec:conver,asu:assumption 4}.

Then, following the first equation in \cite[Section 3.1]{friedlander_erratum_2011}, we bound $||\bm{e}_{n,2}^{(t)}||^2$ as
\begin{align}\label{error:e2}
	\left\|\bm{e}_{n,2}^{(t)}\right\|^{2}\leq\frac{4}{{K_n}^{2}}&\left(K_n-\sum_{m \in \mathcal{M}^{(t)}} K_{nm}\right)^{2} \\	
	&\nonumber\times\left(\beta_{n,1}+\beta_{n,2}\left\|\nabla_n \mathcal{L}_n\left(\bm{\theta}_n^{(t)}\right)\right\|^{2}\right),
\end{align}
where $K_n=\sum_{m=1}^MK_{nm}$.

Combining (\ref{sec:system,ssec:conve,equ:error_all}), (\ref{sec:system,ssec:conve,equ:errorn_all}), (\ref{expectation error}), (\ref{error:e1}) and (\ref{error:e2}) at the $t$-th training round, we have 		
\begin{align}\label{trainloss_optimumloss}
	\mathcal{L}_n(\bm{\theta}_n^{(t+1)})&\leq \mathcal{L}_n(\bm{\theta}_n^{(t)})+\frac{\beta_{n,1}}{L_n}\Psi_n^{(t)}+\frac{3d_n{v_n^{\star}}^{(t)}}{2L_n} \\\nonumber& -\frac{\|\nabla_n \mathcal{L}_n(\bm{\theta}_n^{(t)})\|^{2}}{2L_n}(1-2\beta_{n,2}\Psi_n^{(t)}),
\end{align}
where $\Psi_n^{(t)}$ is defined in (\ref{define:psi}). From \cite[eq. (2.4)]{friedlander_erratum_2011}, we have $||\nabla \mathcal{L}_n(\bm{\theta}_n^{(t)})||^{2} \geq 2 \Omega_n(\mathcal{L}_n(\bm{\theta}_n^{(t)})-\mathcal{L}_n(\bm{\theta}_n^{(\star)}))$. Subtracting $\mathcal{L}_n(\bm{\theta}_n^{(\star)})$ on both sides of (\ref{trainloss_optimumloss}), applying the above inequality and we obtain
\begin{align}\label{equ:theta_star}
	\nonumber\mathcal{L}_n&(\bm{\theta}_n^{(t+1)})-\mathcal{L}_n(\bm{\theta}_n^{(\star)})\leq
	\left(\mathcal{L}_n(\bm{\theta}_n^{(1)})-\mathcal{L}_n(\bm{\theta}_n^{(\star)})\right)\\ & \times \prod_{t'=1}^{t}\Upsilon_n^{(t')}+ \sum_{t''=1}^{t-1}C_n^{(t'')}\prod_{t'=t''}^{t-1}\Upsilon_n^{(t'+1)} + C_n^{(t)},
\end{align}
where $\Upsilon_n^{(t)}$ and $C_n^{(t)}$ are defined in (\ref{define:psi,upsilon_C}).
Finally, combining (\ref{equ:theta_star}) with (\ref{sec:system,ssec:multi,equ:optimum function}), and we obtain (\ref{equ:final_convergence}), which completes the proof.

\bibliographystyle{IEEEtran}
\bibliography{Output}
\end{document}